\documentclass[10pt,letterpaper]{article}
\usepackage[margin=1in]{geometry}
\usepackage{graphicx}
\usepackage{color}
\usepackage{amsmath,amsfonts}
\usepackage{mymath}
\usepackage{graphicx}
\usepackage{authblk}
\usepackage{algorithmic, algorithm}

\title{Smooth Sparse Coding via Marginal Regression for Learning Sparse Representations }
\author[1]{Krishnakumar Balasubramanian\thanks{krishnakumar3@gatech.edu}}
\author[2]{Kai Yu \thanks{yukai@baidu.com}}
\author[1]{Guy Lebanon\thanks{lebanon@cc.gatech.edu}}
\affil[1]{College of Computing, Georgia Tech.}
\affil[2]{Baidu Inc.}

\begin{document}
\maketitle

\begin{abstract}
We propose and analyze a novel framework for learning sparse representations, based on two statistical techniques: kernel smoothing and marginal regression. The proposed approach provides a flexible framework for incorporating feature similarity or temporal information present in data sets, via non-parametric kernel smoothing. We provide generalization bounds for dictionary learning using smooth sparse coding and show how the sample complexity depends on the $L_1$ norm of kernel function used. Furthermore, we propose using marginal regression for obtaining sparse codes, which significantly improves the speed and allows one to scale to large dictionary sizes easily. We demonstrate the advantages of the proposed approach, both in terms of accuracy and speed by extensive experimentation on several real data sets. In addition, we demonstrate how the proposed approach could be used for improving semi-supervised sparse coding.
\end{abstract}

\section{Introduction}
Sparse coding is a popular unsupervised paradigm for learning sparse representations of data samples, that are subsequently used in classification tasks. In standard sparse coding, each data sample is coded independently with respect to the dictionary. We propose a smooth alternative to traditional sparse coding that incorporates feature similarity, temporal or other user-specified domain information between the samples, into the coding process. 

The idea of smooth sparse coding is motivated by the relevance weighted likelihood principle. Our approach constructs a code that is efficient in a smooth sense and as a result leads to improved statistical accuracy over traditional sparse coding. The smoothing operation, which could be expressed as non-parametric kernel smoothing, provides a flexible framework for incorporating several types of domain information that might be available for the user. For example, for image classification task, one could use: (1) kernels in feature space for encoding similarity information for images and videos, (2) kernels in time space in case of videos for incorporating temporal relationship, and (3) kernels on unlabeled image in the semi-supervised learning and transfer learning settings. 

Most sparse coding training algorithms fall under the general category of alternating procedures with a convex lasso regression sub-problem. While efficient algorithms for such cases exist \cite{wright2008robust,lee2007efficient}, their scalability for large dictionaries remains a challenge. We propose a novel training method for sparse coding based on marginal regression, rather than solving the traditional alternating method with lasso sub-problem. Marginal regression corresponds to several univariate linear regression followed by a thresholding step to promote sparsity. For large dictionary sizes, this leads to a dramatic speedup compared to traditional sparse coding methods (up to two orders of magnitude) without sacrificing statistical accuracy. 

We further develop theory that extends the sample complexity result of ~\cite{vainsencher11} for dictionary learning using standard sparse coding to the smooth sparse coding case. We specifically show how the sample complexity depends on the $L_1$ norm of the kernel function used. 

Our main contributions are: (1) proposing a framework based on kernel-smoothing for incorporating feature, time or other similarity information between the samples into sparse coding, (2) providing sample complexity results for dictionary learning using smooth sparse coding, (3) proposing an efficient marginal regression training procedure for sparse coding, and (4) successful application of the proposed method in various classification tasks. Our contributions lead to improved classification accuracy in conjunction with computational speedup of two orders of magnitude.

\section{Related work} 

Our approach is  related to the local regression method \cite{loader1999local,hastie1993local}. More recent related work is \cite{meier2007smoothing} that uses smoothing techniques in high-dimensional lasso regression in the context of temporal data. Another recent approach proposed by~\cite{yu2009nonlinear} achieves code locality by approximating data points using a linear combination of nearby basis points. The main difference is that traditional local regression techniques do not involve basis learning. In this work, we propose to learn the basis or dictionary along with the regression coefficients locally. 


In contrast to previous sparse coding papers we propose to use marginal regression for learning the regression coefficients, which results in a significant computational speedup with no loss of accuracy.  Marginal regression is a relatively old technique that has recently reemerged as a computationally faster alternative to lasso regression \cite{fan2008}. See also \cite{genovese12} for a statistical comparison of lasso regression and marginal regression. 




\section{Smooth Sparse Coding}\label{sec:ssc}

\textbf{Notations:} The notations $x$ and $X$ correspond to vectors and matrices respectively, in appropriately defined dimensions; the notation  $\|\cdot\|_p$ corresponds to the $L_p$ norm of a vector (we use mostly use $p=1,2$ in this paper); the notation $\|\cdot\|_F$ corresponds to the Frobenius norm of a matrix; the notation $|f|_p$ corresponds to the $L_p$ norm of the function $f$: $(\int |f|^p \, d\mu) ^{1/p}$; the notation $x_i, i=1,\ldots,n$ corresponds to the data samples, where we assume that each sample $x_i$ is a $d$-dimensional vector. The explanation below uses $L_1$ norm for sparsity for simplicity. But the method applies more generally to any structured regularizers, for e.g.,~\cite{bronstein2012learning, jenatton2010proximal}.

The standard sparse coding problem consists of solving the following optimization problem,
\begin{align*}
\min_{\substack{D \in \mathbb{R}^{d \times K}\qquad \\ \beta_i\in \mathbb{R}^K, i=1,\ldots,n}} &\sum_{i=1}^n\|x_i-D\beta_i\|_2^2   \\
\text{subject to} \quad&\quad\|d_j\|_2 \leq 1 \quad j= 1,\ldots K \\
&\quad \|\beta_i\|_1 \leq \lambda \quad  i= 1,\ldots n.
\end{align*}
where $\beta_i\in\mathbb{R}^K$ corresponds to the encoding of sample $x_i$ with respected to the dictionary $D \in \mathbb{R}^{d \times K}$ and $d_j \in \mathbb{R}^d$ denotes the $j$-column of the dictionary matrix $D$. The dictionary is typically over-complete, implying that $K > d$. 

Object recognition is a common sparse coding application where $x_i$ corresponds to a set of features obtained from a collection of image patches, for example SIFT features~\cite{lowe1999object}. The dictionary $D$ corresponds to an alternative coding scheme that is higher dimensional than the original feature representation. The $L_1$ constraint promotes sparsity of the new encoding with respect to $D$. Thus, every sample is now encoded as a sparse vector that is of higher dimensionality than the original representation.

In some cases the data exhibits a structure that is not captured by the above sparse coding setting. For example, SIFT features corresponding to samples from the same class are presumably closer to each other compared to SIFT features from other classes. Similarly in video, neighboring frames are presumably more related to each other than frames that are farther apart. In this paper we propose a mechanism to incorporate such feature similarity and temporal information into sparse coding, leading to a sparse representation with an improved statistical accuracy (for example as measured by classification accuracy). 


We consider the following smooth version of the sparse coding problem above:
\begin{align}
\min_{\substack{D \in \mathbb{R}^{d \times K}\qquad \\ \beta_i\in \mathbb{R}^K, i=1,\ldots,n}}  &\sum_{i=1}^n\sum_{j=1}^n w(x_j,x_i) \|x_j-D\beta_i\|_2^2  \label{eq:ssc1}\\
\text{subject to} \quad&\quad\|d_j\|_2 \leq 1 \quad j= 1,\ldots K \\
&\quad \|\beta_i\|_1 \leq \lambda \quad  i= 1,\ldots n.
\label{eq:ssc3}
\end{align}
where $\sum_{j=1}^n w(x_j,x_i)=1$ for all $i$. It is convenient to define the weight function through a smoothing kernel
\begin{align*}
w(x_j,x_i)= \frac{1}{h_1} \mathcal{K}_1\left(\frac{\rho(x_j,x_i)}{h_1}\right) 
\end{align*}
where $\rho(\cdot,\cdot)$ is a distance function that captures the feature similarity, $h_1$ is the bandwidth, and  $\mathcal{K}_1$ is a smoothing kernel. Traditional sparse coding minimizes the reconstruction error of the encoded samples. Smooth sparse coding, on the other hand, minimizes the reconstruction of encoded samples with respect to their neighbors (weighted by the amount of similarity). 

The smooth sparse coding setting leads to codes that represent a neighborhood rather than an individual sample and that have lower mean square reconstruction error (with respect to a given dictionary), due to lower estimation variance (see for example the standard theory of smoothed empirical process \cite{devroye2001combinatorial}).

\subsection{The choice of smoothing kernel}

There are several possible ways to determine the weight function $w$. One common choice for the kernel function is the Gaussian kernel whose bandwidth is selected using cross-validation. Other common choices for the kernel are the triangular, uniform, and tricube kernels.  The bandwidth may be fixed throughout the input space, or may vary in order to take advantage of non-uniform samples. We use in our experiment the tricube kernel with a constant bandwidth.


The distance function $\rho(\cdot,\cdot)$ may be one of the standard distance functions (for example based on the $L_p$ norm). Alternatively, $\rho(\cdot,\cdot)$ may be expressed by domain experts, learned from data before the sparse coding training, or learned jointly with the dictionary and codes during the sparse coding training.  

\subsection{Spatio-Temporal smoothing}\label{sec:sssc}

In spatio-temporal applications we can extend the kernel to include also a term reflecting the distance between the corresponding time or space
\begin{align*}
w(x_j,x_i) = \frac{1}{h_1}\mathcal{K}_1\left(\frac{\rho(x_j,x_i)}{h_1}\right) \frac{1}{h_2}\mathcal{K}_2\left(\frac{j -i}{h_2}\right).
\end{align*}
Above, $\mathcal{K}_2$ is a univariate symmetric kernel with bandwidth parameter $h_2$. One example is video sequences, where the kernel above combines similarity of the frame features and the time-stamp. 

Alternatively, the weight function can feature only the temporal component and omit the first term containing the distance function between the feature representation. A related approach for that situation, is based on the Fused lasso which penalizes the absolute difference between codes for neighboring points. The main drawback of that approach is that one needs to fit all the data points simultaneously whereas in smooth sparse coding, the coefficient learning step decomposes as $n$ separate problems which provides a computational advantage (see Section~\ref{sec:temporalexp} for more details). Also, while fused Lasso penalty is suitable for time-series data to capture relatedness between neighboring frames, it may not be immediately suitable for other situations that the proposed smooth sparse coding method could handle.





\section{Marginal Regression for Smooth Sparse Coding}\label{sec:sscalgo}

A standard algorithm for sparse coding is the alternating bi-convex minimization procedure, where one alternates between (i) optimizing for codes (with a fixed dictionary) and (ii) optimizing for dictionary (with fixed codes). Note that step (i) corresponds to regression with $L_1$ constraints and  step (ii) corresponds to least squares with $L_2$ constraints.  In this section we show how marginal regression could be used to obtain better codes faster (step (i)). In order to do so, we first give a brief description of the marginal regression procedure. 

\textbf{Marginal Regression:} Consider a regression model $y= X \beta + z$ where $y \in \mathbb{R}^n$, $\beta \in \mathbb{R}^p$, $X \in \mathbb{R}^{n \times p}$ with $L_2$ normalized columns (denoted by $x_j$), and $z$ is the noise vector. Marginal regression proceeds as follows: 
\begin{itemize}
\item  Calculate the least squares solution \[\hat \alpha^{(j)} = x_j^T y.\]
\item  Threshold the least-square coefficients \[\hat \beta^{(j)} = \hat \alpha^{(j)} 1_{\{|\hat \alpha^{(j)}| >t\}}, \quad j=1,\ldots, p.\]  
\end{itemize}

Marginal regression requires just $O(np)$ operations compared to $O(p^3+np^2)$, the typical complexity of lasso algorithms. When $p$ is much larger than $n$, marginal regression provides two orders of magnitude over Lasso based formulations. Note that in sparse coding, the above speedup occurs for each iteration of the outer loop, thus enabling sparse coding for significantly larger dictionary sizes. Recent studies have suggested that marginal regression is a viable alternative for Lasso given its computational advantage over lasso. A comparison of the statistical properties of marginal regression and lasso is available in~\cite{fan2008,genovese12}. 

Applying marginal regression to smooth sparse coding, we obtain the following scheme. The marginal least squares coefficients are \[\hat \alpha^{(k)}_i = \sum_{j=1}^n \frac{w(x_j,x_i)}{\|d_k\|_2} d_k^T x_j.\] 
We sort these coefficient in terms of their absolute values, and select the top $s$ coefficients whose $L_1$ norm is bounded by $\lambda$: 
\begin{align*}
\hat \beta^{(k)}_i &=
\begin{cases}
\hat \alpha^{(k)}_i & k \in S\\
0 & k \notin S
\end{cases},
\quad\text{where}\\
S &=\left\{1, \ldots , s\,:\, s \leq d: \sum_{k=1}^s|\hat \alpha^{(k)}_i| \le \lambda \right\}
\end{align*}

We select the thresholding parameter using cross validation in each of the sparse coding iterations. Note that the same approach could be used with structured regularizers too, for example \cite{bronstein2012learning, jenatton2010proximal}.

Marginal regression works well when there is minimal correlation between the different dictionary atoms. In the linear regression setting, marginal regression performs much better with orthogonal data~\cite{genovese12}. In the context of sparse coding, this corresponds to having uncorrelated or incoherent dictionaries \cite{tropp2004}. One way to measure such incoherence is using the babel function, which bounds the maximum inner product between two different columns $d_i,d_j$: 
\begin{align*}
\mu_s(D)= \max_{i \in \{1,\ldots,d\}} \max_{\Lambda \subset \{1,\ldots,d\} \backslash \{i\}; |\Lambda|=s} \sum_{j\in \Lambda }|d_j^\top d_i|.
\end{align*}
An alternative, which leads to easier computation is enforcing the constraint $\|D^TD -I_{K\times K}\|_F^2$ when optimizing over the dictionary matrix $D$
\begin{align*}
\hat D &= \argmin_{D\in \mathcal{D}} \sum_{i=1}^n \|x_i - D \hat \beta_i \|_2^2, \quad \text{where}\\
\mathcal{D} &= \{D \in \mathbb{R}^{ d \times K} : \|d_j\|_2^2 \leq 1, \|D^\top D -I\|_F^2 \leq \gamma\}.
\end{align*}

We use the method of optimal directions update \cite{ramirez2009} to solve the above optimization problem. Specifically, representing the constraints using the Lagrangian and setting the derivative with respect to $D$ to zero, we get the following update rule
\begin{align*}
\hat D_{(t+1)} = &\left( \hat B_{(t+1)} {\hat B_{(t+1)}}^\top + 2 \kappa {\hat D_{t}}^\top \hat D_{t} + 2 \eta \text{diag}({\hat D_{t}}^\top \hat D_{t})\right) \\
& \left(X{\hat B_{(t+1)}}^\top+ 2(\kappa+\eta)\hat D_{t}\right).
\end{align*}

Above, $\hat B_t = [\hat \beta_1(t) , \ldots,\hat  \beta_n(t)]$ is the matrix of data codes obtained in iteration $t$, $X \in \mathbb{R}^{p \times n}$ is the data in matrix format, $\kappa$ is a regularization parameter corresponding to the incoherence constraints, and $\eta$ is a regularization parameter corresponding to the normalization constraints. Note that if $\kappa = \eta = 0$, the update reduces to standard least squares update with no constraints. 

A sequence of such updates corresponding to step (i) and step (ii) converges to a stationary point of the optimization problem (this can be shown using Zangwill's theorem~\cite{zangwill1969nonlinear}). But no provable algorithm that converges to the global minimum of the smooth sparse coding (or standard sparse coding) exists yet. Nevertheless, the main idea of this section is to speed-up the existing alternating bi-convex minimization procedure for obtaining sparse representations, by using marginal regression. 

\begin{algorithm}[t]
\caption{Smooth Sparse Coding via Marginal Regression}
\begin{algorithmic}
 \STATE {\bfseries Input:} Data $\{(x_1,y_1),\ldots, (x_n,y_n)\}$  and kernel/similarity measure $\mathcal{K}_1$ and $d_1$.
 \STATE {\bfseries Precompute:} Compute the weight matrix $w(i,j)$ using the kernel/similarity measure and 
\STATE {\bfseries Initialize:} Set the dictionary at time zero to be $D_0$.
 \STATE {\bfseries Algorithm:} \\
  \REPEAT 
 \STATE {\bfseries Step (i):} For all $i=1,\ldots,n$, solve marginal regression:
 \begin{align*}
\hat \alpha^{(k)}_i& = \sum_{j=1}^n \frac{w(x_j,x_i)}{\|d_k\|_2} d_k^T x_j\\
\hat \beta^{(k)}_j &=
\begin{cases}
\hat \alpha^{(k)}_j & j \in S\\
0 & j \notin S
\end{cases},\\
S &=\{1, \ldots , s; s \leq d: \sum_{k=1}^s| \hat \alpha^{(k)}_i| \le \lambda \}.
\end{align*} 
 \STATE {\bfseries Step (ii):} Update the dictionary based on codes from previous step. 
 \begin{align*}
 \hat D_t &= \argmin_{D\in \mathcal{D}} \sum_{i=1}^n \|x_i - D \hat \beta_i (t)\|_2^2, \quad\text{where} \\
\mathcal{D}&= \{D \in \mathbb{R}^{ d \times K} : \|d_j\|_2^2\leq 1, \|D^\top D -I\|_F^2 \leq \gamma\}
\end{align*}
\UNTIL{convergence}
\STATE {\bfseries Output:} Return the learned codes and dictionary. 
\end{algorithmic} \label{alg:ssc}
\end{algorithm}

\section{Sample Complexity of Smooth sparse coding}
In this section, we analyze the sample complexity of the proposed smooth sparse coding framework. Specifically, since there does not exist a provable algorithm that converges to the global minimum of the optimization problem in Equation~\eqref{eq:ssc1}, we provide uniform convergence bounds over the dictionary space and thereby prove a sample complexity result for dictionary learning under smooth spare coding setting. We leverage the analysis for dictionary learning in the standard sparse coding setting by \cite{vainsencher11} and extend it to the smooth sparse coding setting. The main difficulty for the smooth sparse coding setting is obtaining a covering number bound for an appropriately defined class of functions (see Theorem 1 for more details).  

We begin by re-representing the smooth sparse coding problem in a convenient format for analysis. Let $x_1,\ldots,x_n$ be independent random variables with a common probability measure $\mathbb{P}$ with a density $p$. We denote by $\mathbb{P}_n$ the empirical measure over the $n$ samples, and the kernel density estimate of $p$ is defined by \[p_{n,h}(x) = \frac{1}{nh}\sum_{i=1}^n \mathcal{K}\left(\frac{\|x-X_i\|_2}{h}\right).\] 
Let $\mathcal{K}_{h_1}(\cdot) = \frac{1}{h_1}\mathcal{K}_1(\frac{\cdot}{h})$. With the above notations, the reconstruction error at the point $x$ is given by 
\begin{align*}
r_{\lambda}(x)= \int\min_{\beta \in \mathcal{S}_{\lambda}} \|x'-D\beta\|_2  \mathcal{K}_{h_1}(\rho(x,x'))\,d\mathbb{P}_n(x') 
\end{align*}
where \[\mathcal{S}_{\lambda} = \{\beta: \|\beta\|_1 \leq \lambda\}.\] The empirical reconstruction error is 
\begin{align*}
\E_{\mathbb{P}_n}(r) = \iint\min_{\beta \in \mathcal{S}_{\lambda}} \|x'-D\beta\|_2  \mathcal{K}_{h_1}(\rho(x,x'))\,d\mathbb{P}_n(x') \,dx 
\end{align*}
and its population version is 
\begin{align*}
\E_{\mathbb{P}}(r)= \iint\min_{\beta \in \mathcal{S}_{\lambda}} \|x'-D\beta\|_2  \mathcal{K}_{h_1}(\rho(x,x'))\,d\mathbb{P}(x') \,dx .
\end{align*}

Our goal is to show that the sample reconstruction error is close to the true reconstruction error. Specifically, to show $\E_{\mathbb{P}}(r_{\lambda}) \leq (1 +\kappa) \E_{\mathbb{P}_n}(r_{\lambda}) + \epsilon$ where $\epsilon,\kappa\geq 0$, we bound the covering number of the class of functions corresponding to the reconstruction error. We assume a dictionary of bounded babel function, which holds as a result of the relaxed orthogonality constraint used in the Algorithm~\ref{alg:ssc} (see also \cite{ramirez2009}). We define the set of $r$ functions with respect the the dictionary $D$ (assuming data lies in the unit $d$-dimensional ball $\mathbb{S}^{d-1}$) by
\begin{align*}
\mathcal{F}_\lambda = \{r_{\lambda}: \mathbb{S}^{d-1} \to \mathbb{R}: D \in \mathbb{R}^{d \times K}, \|d_i\|_2 \leq  1, \mu_s(D)\leq \gamma\}.
\end{align*}
The following theorem bounds the covering number of the above function class. 
\begin{thm}
For every $\epsilon > 0$, the metric space $(\mathcal{F}_\lambda, |\cdot|_{\infty})$ has a  subset of cardinality at most ${\left(\frac{4\lambda|\mathcal{K}_{h_1}(\cdot)|_1   }{\epsilon(1-\gamma)}\right)}^{dK}$, such that every element from the class is at a distance of at most $\epsilon$ from the subset, where $|\mathcal{K}_{h_1}(\cdot)|_1 = \int |\mathcal{K}_{h_1}(x)|\,d\mathbb{P}$.
\end{thm}
\begin{proof}
Let $\mathcal{F'}_{\lambda} = \{r'_{\lambda} : \mathbb{S}^{d-1} \to \mathbb{R}: D \in {d \times K}, \|d_i\|_2 \leq  1\}$, where $r'_{\lambda}(x) = \min_{\beta \in \mathcal{S}_\lambda} \|D\beta -x\|$. With this definition we note that $\mathcal{F}_\lambda$ is just $\mathcal{F'}_\lambda$ convolved with the kernel $\mathcal{K}_{h_1}(\cdot)$. 
By Young's inequality~\cite{devroye2001combinatorial} we have, \[|\mathcal{K}_{h_1} * (s_1 - s_2)|_p \leq |\mathcal{K}_{h_1}|_1 |s_1-s_2|_p, \quad 1 \leq p \leq \infty\] for any $L_p$ integrable functions $s_1$ and $s_2$. Using this fact, we see that convolution mapping between metric spaces $\mathcal{F'}$ and  $\mathcal{F}$ converts $\frac{\epsilon}{|\mathcal{K}_{h_1}(\cdot)|_1} $ covers into $\epsilon$ covers. From~\cite{vainsencher11}, we have that the the class $\mathcal{F'}_\lambda$ has $\epsilon$ covers of size at most ${(\frac{4\lambda   }{\epsilon(1-\gamma)})}^{dK}$. This proves the the statement of the theorem. 
\end{proof}

This leads to the following generalization bound for the smooth sparse coding. 
\begin{thm}
Let $\gamma <1$, $\lambda > e/4$ with distribution $\mathbb{P}$ on $\mathbb{S}^{d-1}$. Then with probability at least $1-e^{-t}$ over the $n$ samples drawn according to $\mathbb{P}$, for all the $D$ with unit length columns and $\mu_s(D) \leq \gamma$, we have:
\begin{align*}
\E_{\mathbb{P}}(r_{\lambda}) \leq \E_{\mathbb{P}_n}(r_{\lambda}) +\sqrt{\frac{dK \ln  \left(\frac{4 \sqrt {n} \lambda |\mathcal{K}_{h_1}(\cdot)|_1}{(1-\gamma)} \right)}{2n}} + \sqrt{\frac{t}{2n}} + \sqrt{\frac{4}{n}}
\end{align*}
\end{thm}

The above theorem follows from the previous covering number bound and the following lemma for generalization bound that is based on the result in \cite{vainsencher11} concerning $|\cdot|_{\infty}$ covering numbers.
\begin{lem}
Let $\mathcal{Q}$ be a function class of $[0,B]$ functions with covering number $(\frac{C}{\epsilon})^d > \frac{e}{B^2}$ under $|\cdot|_{\infty}$  norm. Then for every $t >0$ with 
probability  at least $1-e^{-t}$, for all $q \in \mathcal{Q}$, we have:
\begin{align*}
\E f \leq E_n f + B \left(\sqrt{\frac{d \ln (C\sqrt{n})}{2n}+ \frac{t}{2n}}\right)+ \sqrt{\frac{4}{n}}. 
\end{align*}
\end{lem}

The above theorem, shows that the generalization error scales as  $O(n^{-1/2})$ (assuming the other problem parameters fixed). In the case of $\kappa > 0$, it is possible to obtain faster rates of $O(n^{-1})$ for smooth sparse coding, similar to derivations in~\cite{bartlett2005}. The following theorem gives the precise statement.

\begin{thm}
Let $\gamma < 1$, $\lambda > e/4$, $dK > 20$ and $n \geq 5000$. Then with probability at least $1-e^{-t}$, we have for all $D$ with unit length and $\mu_s(D) \leq \gamma$, 
\begin{align*}
\E_{\mathbb{P}}(r_{\lambda}) \leq 1.1 \E_{\mathbb{P}_n}(r_{\lambda}) +9 \frac{dK \ln \left(\frac{4n \lambda |\mathcal{K}_{h_1}(\cdot)|_1}{(1-\gamma)}\right )+ t}{n} .
\end{align*}
\end{thm}

The above theorem follows from the covering number bound above and Proposition 22 from~\cite{vainsencher11}. The definition of $r_{\lambda}(x)$ differs from (1) by a square term, but it could  easily be incorporated into the above bounds resulting in an additive factor of $2$ inside the logarithm term.

\section{Experiments}
We demonstrate the advantage of the proposed approach both in terms of speed-up and accuracy, over standard sparse coding. A detailed description of all real-world data sets used in the experiments are given in the appendix. 

\subsection{Speed comparison} 
We conducted synthetic experiments to examine the speed-up provided by sparse coding with marginal regression. The data was generated from a a 100 dimensional mixture of two Gaussian distribution that satisfies   $\|\mu_1 - \mu_2\|_2 = 3$ (with identity covariance matrices). The dictionary size was fixed at $1024$. 

We compare the proposed smooth sparse coding algorithm, standard sparse coding with lasso~\cite{lee2007efficient} and marginal regression updates respectively, with a relative reconstruction error $\|X - \hat D \hat B\|_F /\|X\|_F$ convergence criterion. We experimented with different values of the relative reconstruction error (less than 10\%) and report the average time. From Table~\ref{tab:timecomp}, we see that smooth sparse coding with marginal regression takes significantly less time to achieve a fixed reconstruction error. This is due to the fact that it takes advantage of the spatial structure and use marginal regression updates. It is worth mentioning that standard sparse coding  with marginal regression updates performs faster compared to the other two methods that uses lasso updates, as expected (but does not take into account the spatial structure). 

\begin{table}[h!]
\centering
\begin{tabular}{|c|c|}
\hline
 Method & time (sec) \\ \hline \hline
SC+LASSO & 560.4 $\pm$13 \\ \hline
SC+MR &   250.6$\pm$18\\ \hline
SSC+LASSO&  540.2$\pm$12    \\ \hline
SSC+MR&    186.4 $\pm$10 \\ \hline
\end{tabular}
\caption{ Time comparison of coefficient learning in SC and SSC with either Lasso or Marginal regression updates. The dictionary update step was same for all methods.}
\label{tab:timecomp}
\vspace{-0.2in}
\end{table}

\subsection{Experiments with Kernel in Feature space}
We conducted several experiments demonstrating the advantage of the proposed coding scheme in different settings. Concentrating on face and object recognition from static images, we evaluated the performance of the proposed approach along with standard sparse coding and LLC~\cite{yu2009nonlinear}, another method for obtaining sparse features based on locality.  Also, we performed experiments on activity recognition from videos based on both space and time based kernels. As mentioned before all results are reported using tricube kernel.

\subsubsection{Image classification}
We conducted image classification experiments on CMU-multipie, 15 Scene and Caltech-101 data sets. Following~\cite{yang2010supervised} , we used the following approach for generating sparse image representation: we densely sampled $16\times 16$ patches from images at the pixel level on a gird with step size 8 pixels, computed SIFT features, and then computed the corresponding sparse codes over a  $1024$-size dictionary. We used max pooling to get the final representation of the image based on the codes for the patches. The process was repeated with different randomly selected training and testing images and we report the average per-class recognition rates (together with its standard deviation estimate) based on one-vs-all SVM classification. We used cross validation to select the regularization and bandwidth parameters. 

As Table~\ref{tab:table1} indicates, our smooth sparse coding algorihtm resulted in significantly higher classification accuracy than standard sparse coding and LLC. In fact, the reported performance  is better than previous reported results using unsupervised sparse coding techniques \cite{yang2010supervised}.

\begin{table}[h!]
\centering
\begin{tabular}{|c|c|c|c|}
\hline
 & CMU-multipie& 15 scene & Caltech-101 \\ \hline \hline
SC & 92.70$\pm$1.21 & 80.28$\pm$2.12& 73.20$\pm$1.14 \\ \hline
LLC & 93.70$\pm$2.22 & 82.28$\pm$1.98& 74.82$\pm$1.65 \\ \hline
SSC&    94.14 $\pm$2.01 & 84.10$\pm$1.87 & 76.24$\pm$2.15 \\ \hline
\end{tabular}
\caption{ Test set error accuracy for face recognition on CMU-multipie data set (left) 15 scene (middle) and Caltech-101 (right) respectively. The performance of the smooth sparse coding approach is better than the standard sparse coding and LLC in all cases. }
\label{tab:table1}
\end{table}

\begin{table}[h!]
\centering
\begin{tabular}{|c|c|c|}
\hline
 Dictionary size & 15 scene& Caltech-101 \\ \hline \hline
1024 & 84.10$\pm$1.87  &76.24 $\pm$2.15\\ \hline
2048 & 87.43$\pm$1.55 &78.33$\pm$1.43 \\ \hline
4096&    89.53$\pm$2.00  & 79.11$\pm$0.87 \\ \hline
\end{tabular}
\caption{ Effect of dictionary size on classification accuracy using smooth sparse coding and marginal regression on 15 scene and Caltech -101 data set. }
\label{tab:dsize}
\end{table}



\textit{Dictionary size:} In order to demonstrate the use of scalability of the proposed method with respect to dictionary size, we report classification accuracy with increasing dictionary sizes using smooth sparse coding. The main advantage of the proposed marginal regression training method is that one could easily run experiments with larger dictionary sizes, which typically takes a significantly longer time for other algorithms. For both the Caltech-101 and 15-scene data set, classification accuracy increases significantly with increasing dictionary sizes as seen in Table~\ref{tab:dsize}.

\subsubsection{Action recognition:}

We further conducted an experiment on activity recognition from videos with KTH action and YouTube data set (see Appendix). Similar to the static image case, we follow the standard approach for generating sparse representations for videos as in \cite{wang2009evaluation}.  We densely sample $16 \times 16 \times 10$ blocks from the video and extract HoG-3d~\cite{klser2008spatio} features from the sampled blocks. We then use smooth sparse coding and max-pooling to generate the video representation (dictionary size was fixed at $1024$ and cross-validation was used to select the regularization and bandwidth parameters). Previous approaches include sparse coding, vector quantization, and $k$-means on top of the HoG-3d feature set (see \cite{wang2009evaluation} for a comprehensive evaluation). As indicated by Table~\ref{tab:table3}, smooth sparse coding results in higher classification accuracy than previously reported state-of-the-art and standard sparse coding on both datasets (see \cite{wang2009evaluation,liu2009recognizing} for a description of the alternative techniques).

\begin{table}[t]
\centering
\begin{tabular}{|c|c|c|}
\hline
Cited method & SC& SSC \\ \hline \hline
92.10~\cite{wang2009evaluation} & 92.423& 93.549 \\ \hline
71.2 ~\cite{liu2009recognizing}& 72.640& 74.974 \\ \hline
\end{tabular}
\caption{ Action recognition (accuracy) for cited method (left), Hog3d+ SC (middle) and Hog3d+ SSC (right): KTH data set(top) YouTube action dataset (bottom).}
\label{tab:table3}
\end{table}



\subsubsection{Discriminatory power}

In this section, we describe another experiment that contrasts the codes obtained by sparse coding and smooth sparse coding in the context of a subsequent classification task. As in \cite{yulearning}, we first compute the codes in both case based on patches and combine it with max-pooling to obtain the image level representation. We then compute the fisher discriminant score (ratio of within-class variance to between-class variance) for each dimension as measures of the discrimination power realized by the representations.

\begin{figure*}[t]
\centering
\includegraphics[scale=0.5]{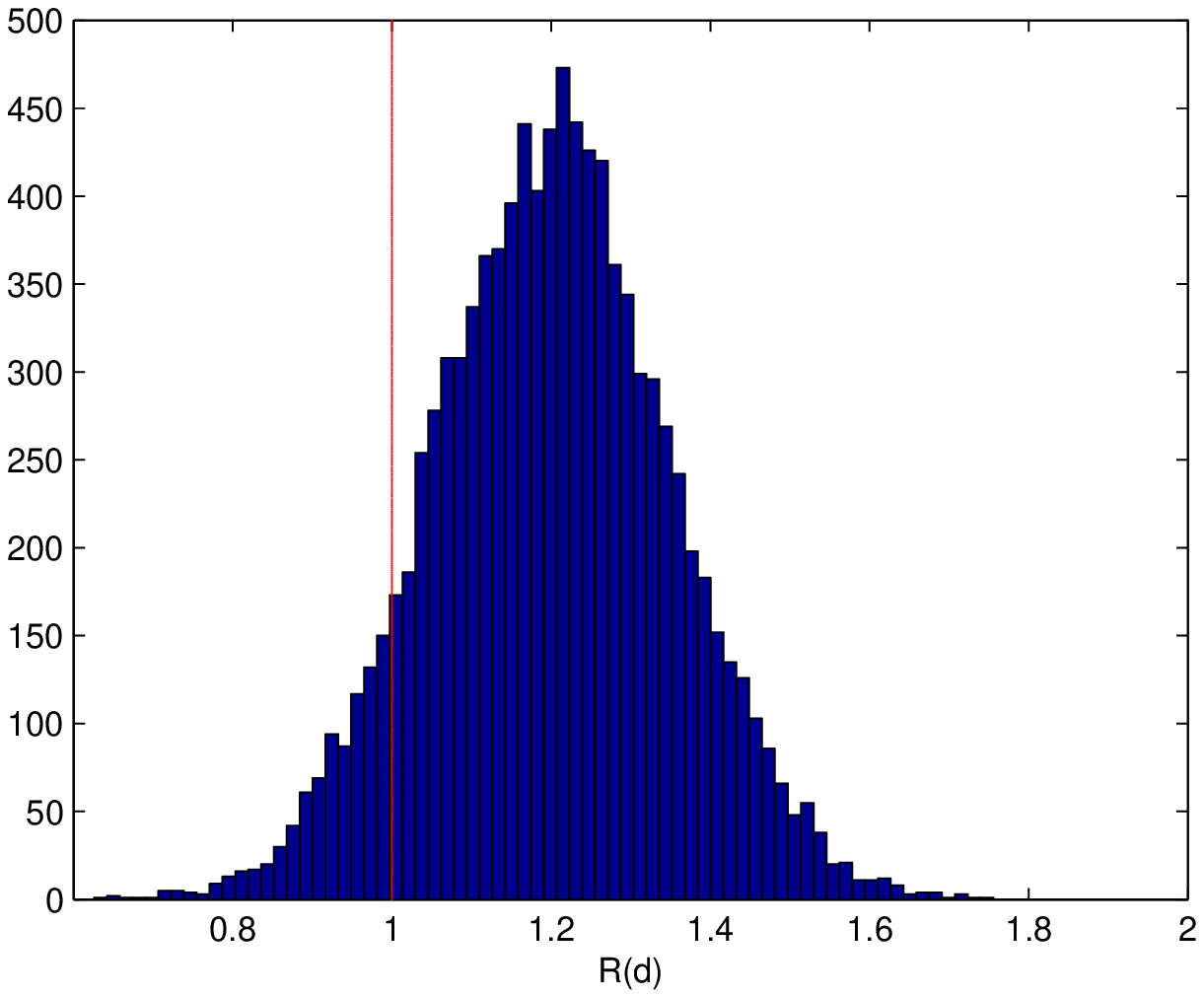}
\includegraphics[scale=0.5]{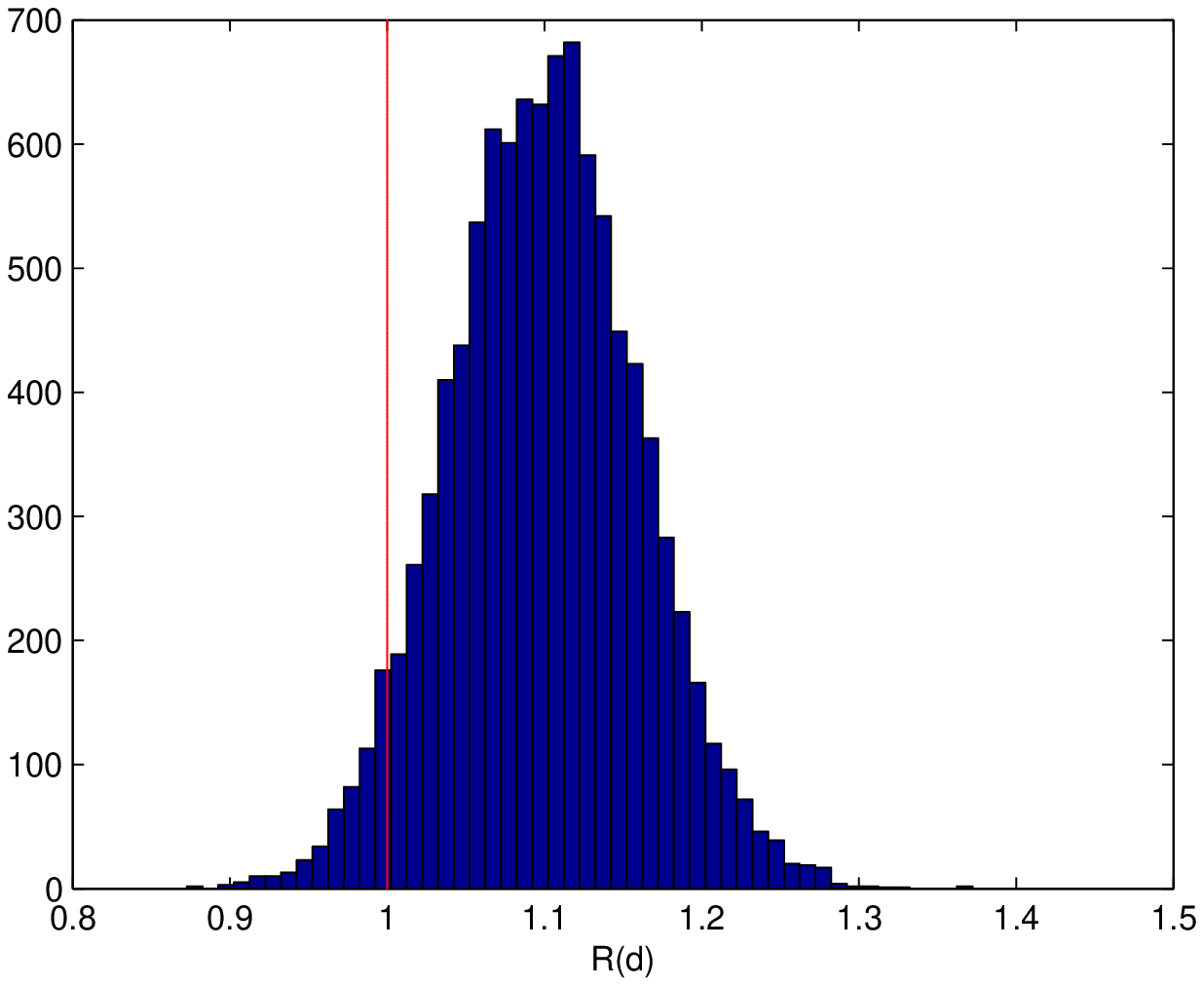}
\caption{Comparison between the histograms of Fisher discriminant score realized by sparse coding and smooth sparse coding. The images represent the histogram of the ratio of smooth sparse coding Fisher score over standard sparse coding Fisher score (left: image data set; right: video). A value greater than 1 implies that smooth sparse coding is more discriminatory.}
\label{fig:disc}
\end{figure*}

Figure~\ref{fig:disc}, graphs a histogram of the ratio of smooth sparse coding Fisher score over standard sparse coding Fisher score $R(d) = F_1(d)/F_2(d)$ for 15-scene dataset (left) and Youtube dataset (right). Both histograms demonstrate the improved discriminatory power of smooth sparse coding over regular sparse coding.

\subsection{Experiments using Temporal Smoothing} \label{sec:temporalexp}
 
In this section we describe an experiment conducted using the temporal smoothing kernel on the Youtube persons dataset. We extracted SIFT descriptors for every $16\times 16$ patches sampled on a grid of step size 8 and used smooth sparse coding with time kernel to learn the codes and max pooling to get the final video representation. We avoided pre-processing steps such as face extraction or face tracking. 
Note that in the previous action recognition video experiment, video blocks were densely sampled and used for extracting HoG-3d features. In this experiment, on the other hand, we extracted SIFT features from individual frames and used the time kernels to incorporate the temporal information into the sparse coding process.

\begin{table}[h]
\centering
\begin{tabular}{|c|c|c|c|}
\hline
Method & Fused Lasso&SC & SSC-tricube   \\ \hline \hline
Accuracy& 68.59 & 65.53 & 69.01\\ \hline
\end{tabular}
\caption{Linear SVM accuracy for person recognition task from YouTube face video dataset.}
\label{tab:table4}
\end{table}

For this case, we also compared to the more standard fused-lasso based approach~\cite{tibshirani2005sparsity}. Note that in fused Lasso based approach, in addition to the standard $L_1$ penalty, an additional $L_1$ penalty on the difference between the neighboring frames for each dimensions is used. This tries to enforce the assumption that in a video sequence, neighboring frames are more related to one another as compared to frames that are farther apart. 

Table~\ref{tab:table4} shows that smooth sparse coding achieved higher accuracy than fused lasso and standard sparse coding. Smooth sparse coding has comparable accuracy on person recognition tasks to other methods that use face-tracking, for example \cite{kim2008face}. Another advantage of smooth sparse coding is that it is significantly faster than sparse coding and the used lasso.

\section{ Semi-supervised smooth sparse coding} 

One of the primary difficulties in some image classification tasks is the lack of availability of labeled data and in some cases, both labeled and unlabeled data (for particular domains). This motivated  semi-supervised learning and transfer learning without labels~\cite{raina2007self} respectively. The motivation for such approaches is that data from a related domain might have some visual patterns that might be similar to the problem at hand. Hence, learning a high-level dictionary based on data from a different domains aids the classification task of interest. 

We propose that the smooth sparse coding approach might be useful in this setting. The motivation is as follows: in semi-supervised, typically not all samples from a different data set might be useful for the task at hand. Using smooth sparse coding, one can weigh the useful points more than the other points (the weights being calculated based on feature/time similarity kernel) to obtain better dictionaries and sparse representations. Other approach to handle a lower number of labeled samples include collaborative modeling or multi-task approaches which impose a shared structure on the codes for several tasks and use data from all the tasks simultaneously, for example group sparse coding~\cite{Bengio2009}. The proposed approach provides an alternative when such collaborative modeling assumptions do not hold, by using relevant unlabeled data samples that might help the task at hand via appropriate weighting.  

We now describe an experiment that examines the proposed smoothed sparse coding approach in the context of semi-supervised dictionary learning. We use data from both CMU multi-pie dataset (session 1) and faces-on-tv dataset (treated as frames) to learn a dictionary using a feature similarity kernel. We follow the same procedure described in the previous experiments to construct the dictionary. In the test stage we use the obtained dictionary for coding data from sessions 2, 3, 4 of CMU-multipie data set, using smooth sparse coding. Note that semi-supervision was used only in the dictionary learning stage (the classification stage used supervised SVM). 

Table \ref{tab:sslssc} shows the test set error rate and compares it to standard sparse coding and LLC~\cite{yu2009nonlinear}. Smooth sparse coding achieves significantly lower test error rate than the two alternative techniques. We conclude that the smoothing approach described in this paper may be useful in cases where there is a small set of labeled data, such as semisupervised learning and transfer learning.

\begin{table}[h!]
\centering
\begin{tabular}{|c|c|c|c|}
\hline
Method &SC & LLC& SSC-tricube  \\ \hline \hline
Test errror & 6.345& 6.003& 5.135 \\ \hline
\end{tabular}
\caption{Semi-supervised learning test set error: Dictionary learned from both CMU multi-pie and faces-on-tv data set using feature similarity kernel, used to construct sparse codes for  CMU multipie data set.}
\label{tab:sslssc}
\end{table}

\section{Discussion and Future work}

We proposed a simple framework for incorporating similarity in feature space and space or time into sparse coding. The codes obtained by smooth sparse coding are significantly more discriminatory than traditional sparse coding, and lead to substantially improved classification accuracy as measured on several different image and video classification tasks.

We also propose in this paper modifying sparse coding by replacing the lasso optimization stage by marginal regression and adding a constraint to enforce incoherent dictionaries. The resulting algorithm is significantly faster (speedup of about two-orders of magnitude over standard sparse coding). This facilitates scaling up the sparse coding framework to large dictionaries, an area which is usually restricted due to intractable computation. We also explore promising extensions to temporal smoothing, semi-supervised learning and transfer learning. We provide bounds on the covering numbers that lead to generalization bounds for the smooth sparse coding dictionary learning problem. 

There are several ways in which the proposed approach can be extended. First, using an adaptive or non-constant kernel bandwidth should lead to higher accuracy. It is also interesting to explore tighter generalization error bounds by directly analyzing the solutions of the marginal regression iterative algorithm. Another potentially useful direction is to explore alternative incoherence constraints that lead to easier optimization and scaling up.

\clearpage
\bibliographystyle{plain} 

\bibliography{papers}

\newpage
\onecolumn
\section{Appendix}
\subsection{Data set Description}

\subsubsection{CMU Multi-pie face recognition:} 
The face recognition experiment was conducted on the CMU Multi-PIE dataset. The dataset is challenging due to the large number of subjects and is one of the standard data sets used for face recognition experiments. The data set contains 337 subjects across simultaneous variations in pose, expression, and illumination. We ignore the 88 subjects that were considered as outliers in~\cite{yang2010supervised} and used the rest of the images for our face recognition experiments.  We follow~\cite{yang2010supervised} and use the 7 frontal extreme illuminations from session one as train set and use other 20 illuminations from Sessions 2-4 as test set.

\subsubsection{15 Scenes Categorization:}
We also conducted scene classification experiments on the 15-Scenes data set. This data set consist of 4485 images from 15 categories, with the number of images each category ranging from 200 to 400. The categories corresponds to scenes from various settings like kitchen, living room etc. Similar to the previous experiment, we extracted patches from the images and computed the SIFT features corresponding to the patches. 

\subsubsection{Caltech-101 Data set:}

The Caltech-101 data set consists of images from 101 classes like animals, vehicles, flowers, etc. The number of images per category varies from 30 to 800. Most images are of medium resolution ($300 \times 300$). All images are used a gray-scale images.  Following previous standard experimental settings for Caltech-101 data set, we use 30 images per category and test on the rest. Average classification accuracy normalized by class frequency is used for evaluation. 

\subsubsection{Activity recognition}

The KTH action dataset consists of 6 human action classes. Each action is performed several times by 25 subjects and is recorded in four different scenarios. In total, the data consists of 2391 video samples. The YouTube actions data set has 11 action categories and is more complex and challenging~\cite{liu2009recognizing}. It has 1168 video sequences of varied illumination, background, resolution etc. We randomly densely sample blocks (400 cuboids) of video from the data sample and extract HOG-3d features and constructed the video features as described above. 

\subsubsection{Youtube person data set} \label{sec:temporalexp}

Similar to the experiments using the feature smoothing kernel, in this section we report results on experiment conducted using the time smoothed kernel. Specifically, we used the YouTube person data set~\cite{kim2008face} in order to recognize people, based on time-based kernel smooth sparse coding. The dataset contains 1910 sequences of 47 subjects. The architecture for this dataset is similar to ~\cite{yang2009linear}. 

\end{document}